\documentclass[]{article}
\usepackage{amssymb}
\usepackage{amsmath}
\usepackage{amsfonts}
\usepackage{graphicx} 
\usepackage{comment}
\usepackage{amsthm}
\usepackage{bbm}
\usepackage{natbib}

\newcommand{\Z}{\mathbb{Z}}

\newcommand{\p}[1]{\ensuremath{\left(#1\right)}}

\newcommand{\N}{\mathbb{N}}
\newcommand{\R}{\mathbb{R}}
\newcommand{\X}{\mathcal{X}}
\newcommand{\Y}{\mathcal{Y}}
\newcommand{\D}{\mathcal{D}}
\newcommand{\norm}[1]{\left\lVert#1\right\rVert}

\newcommand{\Ev}[1]{\mathbb{E}\left[#1\right]}

\newtheorem{theorem}{Theorem}[section]
\newtheorem{corollary}{Corollary}[theorem]
\newtheorem{lemma}[theorem]{Lemma}

\usepackage{PRIMEarxiv}
\usepackage{geometry}
\usepackage{hyperref}
\hypersetup{
    colorlinks=true,
    linkcolor=blue,
    filecolor=magenta,      
    urlcolor=cyan,
    citecolor=black,
    pdftitle={Overleaf Example},
    pdfpagemode=FullScreen,
    }
\begin{document}

\title{Sequence Length Independent Norm-Based Generalization Bounds for Transformers}

\author{ Jacob Trauger  \\Department of Statistics \\ University of Michigan \\ jtrauger@umich.edu \And Ambuj Tewari  \\ Department of Statistics \\ University of Michigan \\tewaria@umich.edu}
\maketitle

\begin{abstract}
    This paper provides norm-based generalization bounds for the Transformer architecture that do not depend on the input sequence length. We employ a covering number based approach to prove our bounds. We use three novel covering number bounds for the function class of bounded linear transformations to upper bound the Rademacher complexity of the Transformer. Furthermore, we show this generalization bound applies to the common Transformer training technique of masking and then predicting the masked word. We also run a simulated study on a sparse majority data set that empirically validates our theoretical findings.
\end{abstract}

\section{INTRODUCTION}
Since \citet{vaswani2017attention} debuted the Transformer, it has become one of the most preeminent architectures of its time. It has achieved state of the art prediction capabilities in various fields \citep{dosovitskiy2020image,wu2022multistep,vaswani2017attention,pettersson2023comparison} and an implementation of it has even passed the BAR exam \citep{katz2023gpt}. With such widespread use, the theoretical underpinnings of this architecture are of great interest. 

Specifically, this paper is concerned with bounding the generalization error when using the Transformer in supervised learning. Upper bounding this can be used to help understand how sample size needs to scale with different architecture parameters and is a very common theoretical tool to understand machine learning algorithms \citep{kakade2008complexity,garg2020generalization,truong2022rademacher,lin2019generalization}.

{\em The main contribution of this paper is providing norm-based generalization bounds for the Transformer architecture that have no explicit dependence on input sequence length}. Previously, the best known norm-based generalization bound scaled with the logarithm of the sequence length \citep{edelman2022inductive}. Removing the dependence on sequence length leads to more intuitively appealing bounds since the total number of parameters in the Transformer is independent of input sequence length. Since our bounds are norm-based they have potential to provide meaningful guarantees even in overparameterized regimes where parameter counting bounds might be less meaningful.

We are able to show this by going through the Rademacher complexity of the Transformer and then using three novel linear covering number bounds to bound the Rademacher complexity. Therefore, Section 1 goes over the necessary background needed. Section 2 shows the novel covering number bounds for a linear transformation function class with bounded matrices and inputs. In Section 3 we start dealing specifically with Transformers. Here we show a new Rademacher complexity bound for a single layer Transformer. Section 4 provides details on how our covering number bounds can be used in the multi-layer analysis of \citet{edelman2022inductive} to get a sequence length independent norm-based generalization bound. Section 5 shows how a method of training used in BERT \citep{devlin2018bert} can be reduced to what is studied in this paper. Finally, in Section 7 we show an experiment on a simulated sparse majority data set to empirically validate our theoretical findings.

\subsection{Related works}
This paper is most closely related to the work of \citet{edelman2022inductive} who prove a norm-based generalization bound that grows logarithmically with sequence length. Due to this, they state Transformers have an inductive bias to represent a sparse function of the inputs. We bolster this claim further by removing the dependence on sequence length altogether. 

Another result similar to our is given by \citet{zhang2022analysis}. They are able to remove the dependence on sequence length. However, the bound they get, which we shall call a parameter counting-based bound, has several drawbacks which we discuss in Section~\ref{sec:twotypes}. \citet{wei2022statistically} also show generalization bounds for Transformers, but specifically study binary classification setting with $0$-$1$ loss and use a margin approach. \citet{fu2023can} freeze some of the weight matrices at initialization and bound the excess risk in this setting as a function of the amount of heads in the attention layer. This paper's bound also do not depend on sequence length.

Outside of Transformers, using Rademacher complexity in deep learning to bound the generalization gap has a rich history. \citet{golowich2018size} was able to use Rademacher complexities to get a generalization bound independent of the depth and width of a neural network. \cite{truong2022rademacher} is able to use Rademacher complexity to get nearly tight bounds on neural networks under some assumptions on the data. Also, \cite{bartlett2017spectrally} use covering numbers and Rademacher complexity to get generalization bounds on multiclass neural networks using a margin based approach.

\section{BACKGROUND}

\subsection{Matrix Definitions}
First, we will define a few well-used matrix norms. Let $p,q,r,d \in \N$ and let $W \in \R^{r \times d}$. The first norm, denoted as $\norm{W}_{q,p}$, will be defined as $\norm{W}_{q,p} = \norm{[\norm{W_{:,1}}_q, \dots \norm{W_{:,d}}_q]^T}_p$ where $W_{:,i}$ is the $i^{\text{th}}$ column of $W$.

The second, also known as the operator norm, we will denote as $\norm{W}_{q \rightarrow p}$. This one is defined as:
\[\norm{W}_{q \rightarrow p} = \sup \frac{\norm{Wx}_p}{\norm{x}_q}\]

One final matrix norm we will review is the Frobenius norm, denoted as $\norm{W}_F$, which is defined as:
\[\norm{W}_F = \sqrt{\sum_{i=1}^{r}\sum_{j=1}^{d}W_{ij}^2}\]

We will also denote $\norm{W}_{2\rightarrow 2}$ as $\norm{W}_2$. A well known property of the $\norm{\cdot}_{2 \rightarrow 2}$ operator is that it is equal to the largest singular value of the input matrix. The Frobenius norm is also well known to be equal to the the square root of the squared sum of the singular values of a matrix. Therefore, we have $\norm{W}_{2 \rightarrow 2} \leq \norm{W}_F$ for any matrix $W$.

\subsection{Generalization Bounds}
When training machine learning algorithms, we can only use a finite amount of data to learn from, however, we want our resulting function to generalize well outside of our training sample. Thus, having guarantees with high probability on the difference between the loss on our training sample and the loss on our testing population is extremely important. Generalization bounds try to upper bound this loss gap.

Mathematically, if we have a hypothesis class $\mathcal{H}$, sample space $\X$, label space $\Y$, loss function $\ell$, and distribution over the sample and label space $\D$, then our generalization gap for a set of samples and labels $S = \{(x_i,y_i)\}_{i=1}^{n}$, $x_i \in \X$, $y_i \in \Y$, on the hypothesis $h \in \mathcal{H}$ is defined to be 
\[\left|\mathbb{E}_{(x,y) \sim \D}\left[\ell(h(x), y)\right] - \frac{1}{n}\sum_{i=1}^{n}\ell(h(x_i), y_i)\right|\]

Notice how if we can have this value go to 0 with high probability over all sets of samples and for all $h \in \mathcal{H}$, then we can be confident that minimizing the sample loss will not impact our generalization.

\subsection{Rademacher Complexity}

One such tool that can be used to upper bound the generalization gap is the Rademacher complexity. Let us have the same set up as in the previous section. Then the Rademacher complexity of a hypothesis class $\mathcal{H}$ is defined to be
\[Rad_n(\mathcal{H}, S) = \frac{1}{n}\mathbb{E}_{\sigma}\left[\sup_{h \in \mathcal{H}}\sum_{i=1}^{n}\sigma_i h(x_i)\right]\]
where each $\sigma_i$ are i.i.d. and take values $\pm 1$ each with half probability and $\sigma = (\sigma_1, \dots, \sigma_n)$. It is well known that \citep{SSS}, if the magnitude of our loss function is bounded above by $c$, with probability greater than $1-\delta$ for all $h \in \mathcal{H}$, we have
\begin{align*}
&\left|\mathbb{E}_{(x,y) \sim \D}\left[\ell(h(x), y)\right] - \frac{1}{n}\sum_{i=1}^{n}\ell(h(x_i), y_i)\right| \leq\\
&2Rad_n(\ell \circ \mathcal{H}, S) + 4c\sqrt{\frac{2\log(4/\delta)}{m}}
\end{align*}

where $\ell \circ \mathcal{H} = \{\ell(h(x), y) \mid (x,y) \in \mathcal{X} \times \mathcal{Y}, h \in \mathcal{H}\}$. Therefore, if we have an upper bound on the Rademacher complexity, we can have an upper bound on the generalization gap.

\subsection{Covering Numbers}
 The use of covering numbers is one such way we can bound the Rademacher complexity of a hypothesis class. Let $q \in \R_{> 0}$ and let us have a function class $\mathcal{F}$, $\forall f \in \mathcal{F}$ $f: \R^d \rightarrow \R^k$. We say a subset $\hat{\mathcal{F}} \subset \mathcal{F}$ covers a set of inputs $\{x_i\}_{i=1}^{n}$ if for every $f \in \mathcal{F}$, $\exists \hat{f} \in \hat{\mathcal{F}}$ such that $\sup_{x_i} \norm{f(x_i) - \hat{f}(x_i)}_q < \epsilon$. We will use the notation $N_{\infty}(\mathcal{F}, \epsilon, \{x_i\}_{i=1}^{n}, \norm{\cdot}_q)$ for this. We will define the covering number, $N_{\infty}(\mathcal{F}, \epsilon, n, \norm{\cdot}_q)$ as 
\[\sup_{\{x_i\}_{i=1}^{n}} N_{\infty}(\mathcal{F}, \epsilon, \{x_i\}_{i=1}^{n}, \norm{\cdot}_q)\]

It has been shown that for scalar valued hypothesis classes, the Rademacher complexity can be upper bounded using the covering number of the hypothesis class \citep{dudley1967sizes}. Using a slightly modified version, we have for a constant $c$:
\[Rad_n(f,S) \leq c \inf_{\delta\geq 0} \p{\delta + \int_{\delta}^{\infty} \sqrt{\frac{\log N_{\infty}(\mathcal{F}, \epsilon, n)}{n}}d\epsilon} \]
Notice that if we have a bounded function class, then the $\infty$ in the integral limit can become the upper bound of the function class.

\subsection{Two Types of Bounds}\label{sec:twotypes}
Suppose our inputs have dimension $d$ and the inputs have an upper norm bound of $B_x$. Suppose our matrices have dimension $k \times d$ and have an upper norm bound of $B_w$. Here we will note two different types of generalization bounds that can one can arrive at depending on the perspective. One is where you have the norm bounds inside a log term and the dimensions on the outside (e.g., $O(\sqrt{dk\log(B_xB_w))/n}$). The other is where you have the bounds outside the log term and parameters inside (e.g., $O(\sqrt{B_xB_w\log(dk))/n}$). We will refer to the former as parameter-counting type bounds and the latter as the norm-based type bounds. We note that these are not the only types of bounds, just these are the ones relevant to this paper.

Notice how, for parameter counting bounds, over-parameterized architectures can lead to vacuous bounds. Also notice that parameter counting bounds do not take full advantage of SGD's implicit regularization since the norms are within the logarithm. In contrast, as long as the norm bounds are reasonable, norm-based generalization bounds work well with over-parameterized architectures and work well with implicit regularization.

For Transformers, there are norm-based bounds that scale logarithmically with sequence length \citep{edelman2022inductive} and parameter counting bounds that do not scale with sequence length \citep{zhang2022analysis}. 
This is a gap in the literature which this paper intends to fill. 

\subsection{Self-Attention and Transformers}
We will follow the definition of self-attention and Transformers set forth by \citet{edelman2022inductive} and keep with their notation. 

Let $X \in \R^{T\times d}$ be the input and $W_c \in \R^{k \times d}$, $W_v \in \R^{d\times k}$, and $W_{Q},W_{K} \in \R^{d \times T}$ be trainable weight matrices. Also, let $\sigma$ be a $L_\sigma$ Lipshitz activation function that is applied elementwise and has the property $\sigma(0)=0$. Finally, let RowSoftmax donote running softmax on each row of its input. Then, they define a Transformer head as 
\[\sigma\p{\text{RowSoftmax}\p{XW_{Q}W_{K}^TX^T}XW_v}W_c\] 
Since $W_Q$ and $W_K$ are only ever multiplied with each other, we will combine $W_{Q}W_{K}^T$ into a single matrix $W_{QK} \in \R^{d \times d}$ for convenience of analysis. Once we do this, note that the total dimensionality (of $W_{QK}, W_c, W_v$) is independent of $T$, the sequence length which counts how many values are in each sample. The embedding dimension is $d$ since it is the dimension that the values in the sequence are embedded into and $k$ is the hidden dimension.

For multi-head Transformers, we assume each head is summed up at the end of each layer. That is, the output for a layer of a multi-head Transformer is:
\[\sum_{h=1}^{H}\sigma\p{\text{RowSoftmax}\p{XW_{h,Q}W_{h,K}^TX^T}XW_{h,v}}W_{h,c}\] 

Note how the ouput of a layer can be used as the input to another layer. This is how multi-layer Transformer networks are created. Standard practice is to add layer normalization in between each layer as this has been well studied to aid in optimization and generalization \citep{ba2016layer,wang2019learning,xu2019understanding}. Thus, keeping with the definitions and notation previously set forth, we will inductively define an $L$-layer Transformer block:

Let $\mathcal{W}^{(i)} = \{W_v^{(i)}, W_c^{(i)}, W_{QK}^{(i)}\}$ and let $\mathcal{W}^{1:i} = \{\mathcal{W}^{1}, \dots \mathcal{W}^{(i-1)}\}$. Also, let 
\[f(X;W^{(i)}) = \sigma\p{\text{RowSoftmax}\p{XW_{Q}W_{K}^TX^T}XW_v}\]
\[g_{block}^{1}\p{X; W^{1:1}} = X\]
Then, the output of the $i^\text{th}$ layer is defined to be
\begin{align*}
    &g_{block}^{(i+1)}\p{X; W^{1:i+1}} =\\
    &\Pi_{norm}\p{\sigma\p{\Pi_{norm}\p{f\p{g_{block}^{(i)}\p{X; W^{1:i}};W^{(i)}}}}W_c^{(i)}}
\end{align*}
where $\Pi_{norm}$ projects each row onto the unit ball.

In our analysis we will focus on the scalar output setting for Transformers. Specifically, we will follow how BERT is trained \citep{devlin2018bert}. In order to get a scalar output, we add an extra input in the sequence that can be constant or trained. Let this index be the $[CLS]$ index. Let us also have a trainable vector $w \in \R^d$. Then, at the last layer, take the output at the $[CLS]$ index, $Y_{[CLS]} \in \R^d$, and we multiply it with $w$ to get our output $w^TY_{[CLS]} \in \R$.

\section{LINEAR COVERING NUMBER BOUNDS}
In this section we will show three different covering number bounds for linear function classes with different restrictions on input and matrix norms. To show the first one, we need the following lemmas, the first one is attributed to Maurey \citep{pisier1981remarques} and first used in this context by \citet{zhang2002covering}:

\begin{lemma} \label{Maurey}
(Maurey's Sparsification Lemma) Let $\mathcal{H}$ be a Hilbert space and let each $f \in \mathcal{H}$ have the representation $f = \sum_{i=1}^{d}\alpha_iV_i$ where $V_i \in \mathcal{H}$, $\norm{V_i} \leq b$ and $\alpha_i \geq 0$ with $\gamma = \norm{\alpha}_1 \leq 1$. Then, for any $k \in \N$, there exist $k_1,\dots, k_d$, $k_i \in \Z_{\geq 0}$, $\sum_{i=1}^{d}k_i \leq k$, such that 
\[\norm{f - \frac{1}{k}\sum_{i=1}^{d}k_iV_i}_2^2 \leq \frac{\gamma^2 b^2 - \norm{f}^2}{k}\]
\end{lemma}
We note that the total amount of $(k_1,\dots k_d)$'s that fit the criteria above is less than or equal to $d^{k}$. This has been used to upper bound the covering number for linear functions \citep{zhang2002covering,bartlett2017spectrally} and we will use it similarly in our proofs.

For covering a class of scalar valued linear functions $\{x\rightarrow w^\top x \mid w \in \R^d, $w$ \text{ norm bounded}\}$ with inputs $\{x_i \in \R^d\}_{i=1}^{n}$,  it is known that we are able to remove the dependence on $n$ and replace it with a dependence on $d$. \citet{kontorovich2021fat}
show this and attribute it to folklore. Given a cover such as this, it is an immediate extension to create a non-$n$-dependent cover of the function class $\{x \rightarrow Wx, W \in \mathcal{W}\}$ as long as each row in each $W \in \mathcal{W}$ is bounded by the norm restraint needed for the scalar valued cover (specific details are in appendix section \ref{1d to nd}).

The next lemma generalizes the above by allowing us to consider more possible norm bounds on $\mathcal{W}$. It shows that under certain norm restrictions for our input, the linear transformation covering number for any set of size $N$ is equivalent to the covering number on the appropriately scaled standard basis (proof in appendix section \ref{proof_equiv_n_d}).

\begin{lemma} \label{equiv_n_d}
    Let $\mathcal{W} \subset \R^{k \times d}$ and let $\mathcal{F} = \{x \rightarrow Wx \mid W \in \mathcal{W}\}$ with $\norm{x}_1 \leq B_x$. Then, for $N \geq d$, we have
    \[\mathcal{N}_\infty\p{\mathcal{F}, \epsilon, N, \norm{\cdot}_q} = \mathcal{N}_\infty\p{\mathcal{F}, \epsilon, \{B_xe_1,\dots,B_xe_d\}, \norm{\cdot}_q}\]
    
\end{lemma}

\subsection{Log Covering Number For $\norm{\cdot}_1$ Bounded Input and $\norm{\cdot}_{1, \infty}$ Bounded Matrix}

Now, we will show our three covering number bounds. Each of the three have different norm bounds on the inputs and the matrices, allowing for flexibility when deciding which to use.

\begin{lemma}\label{1,1,inf}
    Let $N \geq d$, $\mathcal{W} = \{W \in \R^{k \times d}\mid \norm{W}_{1,\infty} \leq B_w\}$, $\mathcal{F} = \{x \rightarrow Wx \mid W \in \mathcal{W}\}$, and let our inputs $x \in \R^d$ have the restriction $\norm{x}_1 \leq B_x$. Then:
    \[\log \mathcal{N}_\infty\p{\mathcal{F}, \epsilon, N, \norm{\cdot}_2} \leq \frac{dB_w^2B_x^2}{\epsilon^2}\log(2k+1)\]
\end{lemma}
\begin{proof}
We will abuse notation slightly by referring to $\mathcal{W}$ at times instead of $\mathcal{F}$ (and similar for $\hat{\mathcal{W}}$ and $\hat{\mathcal{F}}$).

Let us have the set $V = \{v \in \R^{k} \mid \norm{v}_1 \leq B_w\}$. Then notice by lemma $\ref{Maurey}$ we have that there exists an $\epsilon/B_x$ cover for $V$ that has log size 
\[\frac{B_w^2B_x^2}{\epsilon^2}\log(2k+1)\]
Let this cover be $\hat{V}$. We claim that 
\[\hat{\mathcal{W}} = \underbrace{\hat{V} \otimes \hat{V} \dots \otimes \hat{V}}_{\text{$d$ total times}}\]
is a cover for $\mathcal{W}$. 

To show this, let $W \in \mathcal{W}$ and let $\hat{W} \in \hat{\mathcal{W}}$ be the one where each column is the vector that would be chosen to cover the corresponding column in $W$. Then, notice for all $i \in [d]$ we have
\[\norm{(W-\hat{W})B_xe_i} = B_x\norm{W_{:,i} - \hat{W}_{:,i}} \leq B_x \frac{\epsilon}{B_x} = \epsilon\]
Therefore 
\[\sup_{i \in [d]}\norm{(W-\hat{W})B_xe_i} \leq \epsilon\]
which, by lemma \ref{equiv_n_d} shows that 
\[\log \mathcal{N}_\infty\p{\mathcal{F}, \epsilon, N, \norm{\cdot}_2} \leq \frac{dB_w^2B_x^2}{\epsilon^2}\log(2k+1)\]
\end{proof}

\subsection{Log Covering Number For $\norm{\cdot}_1$ Bounded Input and $\norm{\cdot}_{2, 1}$ Bounded Matrix}
\begin{lemma} \label{1,2,1}

Let $N>d$, $\mathcal{W} = \{W \in \R^{k \times d}\mid \norm{W}_{2,1} \leq B_w\}$, $\mathcal{F} = \{x \rightarrow Wx \mid W \in \mathcal{W}\}$, and let our inputs $x \in \R^d$ have the restriction $\norm{x}_1 \leq B_x$. Then:
 \[\log \mathcal{N}_\infty\p{\mathcal{F}, \epsilon, N, \norm{\cdot}_2} \lesssim \frac{B_w^2B_x^2}{\epsilon^2}\log(dk)\]
where $\lesssim$ hides logarithmic dependencies except $T,k,d$.
\end{lemma}
\begin{proof}

This proof uses Lemma 4.6 in Edelman et al., 2022 \cite{edelman2022inductive}, which is rewritten below for clarity.

\begin{lemma}(Edelman et al., 2022 Lemma 4.6)
    Let $\mathcal{W}$ and $\mathcal{F}$ be as above. Then for any set of points $x_1,\dots,x_n \in R^d$ with $\norm{x_i}_2 \leq B_x$ for all $i$, we have 
    \[\log \mathcal{N}_\infty\p{\mathcal{F}, \epsilon, \{x_i\}_{i=1}^{n}, \norm{\cdot}_2} \lesssim \frac{B_w^2B_x^2}{\epsilon^2}\log(dn)\]
\end{lemma}

 With this, let $\hat{\mathcal{W}}$ be an  $\epsilon$-cover for $\mathcal{W}$ over the inputs $\{B_xe_i\}_{i=1}^{d}$ as stated in the lemma. Then, by lemma \ref{equiv_n_d}, we have that the cardinality of $\hat{\mathcal{W}}$ is also an upper bound for $\mathcal{N}_\infty\p{\mathcal{F}, \epsilon, N, \norm{\cdot}_2}$, which is what we needed to show.
\end{proof}

\subsection{Log Covering Number For $\norm{\cdot}_2$ Bounded Input and $\norm{\cdot}_{1, 1}$ Bounded Matrix}
\begin{lemma}\label{2,1,1}
    Let $N\geq d$, $\mathcal{W} = \{W \in \R^{k \times d}\mid \norm{W}_{1,1} \leq B_w\}$, $\mathcal{F} = \{x \rightarrow Wx \mid W \in \mathcal{W}\}$, and let our inputs $x \in \R^d$ have the restriction $\norm{x}_2 \leq B_x$. Then:
     \[\log \mathcal{N}_\infty\p{\mathcal{F}, \epsilon, N, \norm{\cdot}_2} \leq \frac{B_x^2B_w^2}{\epsilon^2}\log(2dk+1)\]
\end{lemma}

\begin{proof}
Let $\mathcal{V}$ be the set of all the flatten matrices in $\mathcal{W}$. Note how this implies $\forall v \in \mathcal{V}$, we have $\norm{v}_1 \leq B_w$. Then, by Maurey's sparification lemma, we have that there exists an $\epsilon$-cover $\hat{\mathcal{V}}$ of log size at most $\frac{B_w^2}{\epsilon^2}\log(2dk+1)$. We claim if we unflatten $\hat{\mathcal{V}}$ (call this $\hat{\mathcal{W}}$), then $\hat{\mathcal{W}}$ is a $(B_x\epsilon)$-cover for $\mathcal{W}$. Let $W \in \mathcal{W}$, let $V$ be the flatten version of $W$. Then, let $\hat{V}$ be the flattened vector we would choose for $V$ in our cover and let $\hat{W} \in \hat{\mathcal{W}}$ be the unflattened version of $\hat{V}$. Notice for any $x \in \R^d$, $\norm{x}_2 \leq B_x$:
\begin{align*}
    &\norm{Wx - \hat{W}x}_2 \leq \norm{W - \hat{W}}_{2\rightarrow 2} \norm{x}_2 \leq\\
    &\norm{W - \hat{W}}_{F} \norm{x}_2 = \norm{V - \hat{V}}_{2} \norm{x}_2 \leq\\
    &\norm{V - \hat{V}}_{2} B_x \leq B_x\epsilon
\end{align*}
Therefore our covering number is at most $\frac{B_x^2B_w^2}{\epsilon^2}\log(2dk+1)$
\end{proof}

\subsection{Observations on Results}
Above we have showed a few different sharpenings of linear covering numbers with matrices instead of vectors. Specifically, these do not rely on the sample size of the input. Also, all but lemma \ref{1,1,inf} keeps the matrix dimensions inside the log term. We do note however, the matrix bound in lemma \ref{1,1,inf} is a $1,\infty$ norm bound while the others are rather $2,1$ or $1,1$ norm bound. We know that for any matrix $W$, $\norm{W}_{p,1} \leq d\norm{W}_{p,\infty}$. Thus if we were to convert lemmas \ref{1,2,1} and \ref{2,1,1} into a norm bound on $2,\infty$ and $1,\infty$ bounds we would have a $d^2B_w^2$ term. This shows that lemma \ref{1,1,inf} is a stronger bound than it lets on.

\section{TRANSFORMER RADEMACHER COMPLEXITY}

\subsection{Analysis for Single Layer Transformer}
Let $w \in R^{d}$, $W_c \in \R^{k\times d}$, $W_v \in \R^{d \times k}$, $W_{QK} \in \R^{d \times d}$ $\norm{w}_1 \leq B_w$, $\norm{W_c^T}_{1,\infty} \leq B_{W_c}$, and $\norm{W_v^T}_{1,\infty} \leq B_{W_v}$. Then we have our scalar one layer Transformer as $w^T Y_{[CLS]}$ where 
\[Y_{[CLS]} = W_c^T\sigma\p{W_v^TX^T\text{softmax}\p{XW_{QK}^Tx_{[CLS]}}}\]

Our Rademacher complexity is thus the following:
{\tiny\[\Ev{\sup_{w, W_c, W_v, W_{QK}} \sum_{i=1}^{m}\epsilon_iw^TW_c^T\sigma\p{W_v^TX_{(i)}^T\text{softmax}\p{X_{(i)}W_{QK}x_{[CLS]}}}}\]}

With this, we have the following theorem:

\begin{theorem}\label{single thoerem}
    Suppose we have a log covering number in the form of $C/\epsilon^2$ for the function class $\{x \rightarrow Wx \mid x \in \mathcal{X}, w \in \mathcal{W}\}$ where $\norm{x}_\infty < B_x$ $\forall x \in \mathcal{X}$. Suppose we also have the norm restrictions above, have $W_{QK}$ meet the covering norm restrictions, and have $m > d$ and $m > \ln(2d)$. Then, an upper bound on the Rademacher complexity of a single layer Transformer layer is:
    \tiny\begin{align*}
    O\left(B_wB_{W_c}L_{\sigma}B_{W_v}\left(\frac{2B_x^2\sqrt{C}}{\sqrt{m}}\left(1 + \ln\left(\frac{\sqrt{m}}{2B_x\sqrt{C}}\right)\right)+ B_x\sqrt{\frac{\ln(2d)}{m}}\right)\right)
    \end{align*}
\end{theorem}
The proof is left in appendix section \ref{proof_single_layer} for ease of presentation.

We can now take lemmas \ref{1,1,inf}, \ref{1,2,1}, and \ref{2,1,1} to get bounds on the Rademacher complexity. In the proof it can be seen the only matrix that needs to be covered in $W_{QK}$, thus we will only have a dependence on $d$ and not $k$. We will show one corollary below and leave the rest to appendix section \ref{coro_single}.

\begin{corollary}
    Let us have the requirements needed for Theorem \ref{single thoerem} along with $\norm{x}_{2} \leq B_x$ and $\norm{W_{QK}}_{1,1} \leq B_{W_{QK}}$ Let 
    \[B = B_wB_{W_c}L_{\sigma}B_{W_v}\]
    and let 
    \[\alpha = B_{W_{QK}}\sqrt{2\log(2d^2+1)}\]
    Then we have our Transformer Rademacher complexity being less than
    \small\begin{align*}
        O\left(B\left(\frac{B_x^3\alpha}{\sqrt{m}}\left(1 + \ln\left(\frac{\sqrt{m}}{B_x^2\alpha}\right)\right)+ B_x\sqrt{\frac{\ln(2d)}{m}}\right)\right)
    \end{align*}
\end{corollary}

\subsection{Single Layer Multiple Heads}
Let $H \in \N$ and let $Y_i$, $i \in [H]$ each be a Transformer head. Then notice by linearity of expectation:
\[\Ev{\sup_{Y_1,\dots,Y_H} \sum_{i=1}^{m}\epsilon_iw^T\sum_{j=1}^{H}Y_j} =  \sum_{j=1}^{H}\Ev{\sup_{Y_j} \sum_{i=1}^{m}\epsilon_iw^TY_j}\]
The expectation above is the same as the single layer, thus multiple heads just adds a linear $H$ term to the Rademacher complexity.

\subsection{Multiple Layers}
We have seen that we are able to get sequence length independent Rademacher complexities (and therefore generalization bounds) for a single layer Transformer architecture.  For multiple layers, it suffices to take the proof found in \citet{edelman2022inductive} and slightly rework it so that it will work for an arbitrary linear covering number bound. 

\begin{theorem}\label{multi-layer}
    (Slight Reworking of Theorem A.17 in Edelman et al., 2022) Suppose we have a log covering numbers in the form of $C_{1}/\epsilon^2$ and $C_{B_x}/\epsilon^2$ for the function class $\{x \rightarrow Wx \mid x \in \mathcal{X}, w \in \mathcal{W}\}$ where $\norm{x}_2 \leq 1$ $\forall x \in \mathcal{X}$ and $\norm{x}_2 \leq B_x$ $\forall x \in \mathcal{X}$ respectively. Suppose we also have $\norm{W_c^{(i)\top}}_2 \leq B_{c2}$, $\norm{W_v^{(i)\top}}_2 \leq B_{v2}$, $\norm{W_{QK}^{(i)}}_2 \leq B_{QK2}$, $\norm{w} \leq B_w$ along with them meeting the needed covering number restrictions.
    Let \small\begin{align*}
        &\alpha_i = \prod_{j=i+1}^{L}L_\sigma B_{c2} B_{v2}(1 + 4B_{QK2})\\
        &\tau_i = \alpha_i^{2/3} + \p{2\alpha_iL_\sigma B_{c2}B_{v2}}^{2/3} + (\alpha_iL_\sigma B_{v2})^{2/3}\\
        &\gamma = C_{B_x}^{1/3}\p{2L_\sigma B_{c2}B_{v2}\alpha_1B_w}^{2/3} + C_1^{1/3}\p{1 + (B_wL_\sigma B_{v2})^{2/3}}\\
        &\eta = C_1^{1/3}\p{B_w^{2/3}\sum_{i=2}^{L}\tau_i}
    \end{align*}
    Then, the log covering number of $g_{scalar}^{L+1}$ is 
    \begin{align*}
        &\frac{(\gamma + \eta)^3}{\epsilon^2}
    \end{align*}
\end{theorem}
The proof is left in appendix section \ref{proof_multi_layer} for ease of presentation.

Notice this is the covering number of the entire multi-layer Transformer. Thus, we can recover an upper bound for the Rademacher complexity of it by using Dudley's integral.

Substituting our covering number bounds into theorem \ref{multi-layer} gives us the three corollaries. We state one below and leave the rest to appendix section \ref{coro_multi_layer}.

\begin{corollary}
     Suppose we have the norm bounds required in lemma \ref{2,1,1} for each $W_c^{(i)}, W_{v}^{(i)}, W_{QK}^{(i)}, w$ and let the maximum be $B$. Let $B_x$ be the input bound. Suppose we also have the bounds needed for theorem \ref{multi-layer}. 
    Let \small\begin{align*}
        &\alpha_i = \prod_{j=i+1}^{L}L_\sigma B_{c2} B_{v2}(1 + 4B_{QK2})\\
        &\tau_i = \alpha_i^{2/3} + \p{2\alpha_iL_\sigma B_{c2}B_{v2}}^{2/3} + (\alpha_iL_\sigma B_{v2})^{2/3}\\
        &\gamma = \p{B^2B_x^2\ln(2dk+1)}^{1/3}\p{2L_\sigma B_{c2}B_{v2}\alpha_1B_w}^{2/3} + \\
        &\p{B^2\ln(2dk+1)}^{1/3}\p{1 + (B_wL_\sigma B_{v2})^{2/3}}\\
        &\eta = \p{B^2\ln(2dk+1)}^{1/3}\p{B_w^{2/3}\sum_{i=2}^{L}\tau_i}
    \end{align*}
    Then, the log covering number of $g_{scalar}^{L+1}$ is 
    \begin{align*}
        &\frac{(\gamma + \eta)^3}{\epsilon^2}
    \end{align*}
\end{corollary}

The above covering number is precise, but unwieldy to look at. To get a better sense of it, we can see that, ignoring polylog terms and constants, we get 
\[B^2B_x^2B_w^2 (L_\sigma B_{c2}B_{v2}B_{QK2})^{O(L)}\frac{1}{\epsilon^2}\]

We do note that the multi-layer method does also work for one layer, however, they have different norm bounds required so they are not quite comparable. If we were to look at just the resulting values, the given single layer method essentially trades the cross product terms in the cubed factor for a factor of $B_x^2$, which seems like an acceptable trade. The proof for the single layer is also much more direct and easy to digest. It also gives a linear dependence on the amount of heads when the multi-layer method extended to multiple heads gives a dependence of $H^{1.5}$.
\section{THEORETICAL EXAMPLE: WORD PREDICTION IN NLP}\label{mask}

Suppose we have a word embedding set up and a vocabulary of size $K$. One way to try to learn is by masking a certain percentage of words in a sentence and asking the Transformer to predict these words. Masking is done by taking the row that corresponds to the position of the masked word (let us call this row $i$) and giving as input a specific vector instead of the actual embedding of the word. Then the prediction is done by taking the vector in row $i$ in the final layer of the Transformer and linearly transforming it into a size $K$ vector. Then we can softmax this vector and use cross entropy loss to train. This is one of the ways BERT \citep{devlin2018bert} is trained. Below, we will suppose only 1 word is masked for each input for ease of presentation. Let us use the cross entropy loss with softmax:

\[\ell_i(y, x) = -\sum_{i=1}^{k}y_i\log(\text{softmax}\p{x})\]

where $y\in \{0,1\}^k$ is a one-hot encoded value that specifies the correct word and $x\in \R^{k}$ is the output of our Transformer at the masked index. It turns out this loss function is $\sqrt{2}$-Lipshitz. The argument comes from the fact the above is differentiable and showing the squared euclidean norm of its gradient is less than $2$.

Therefore, if we let $W$ be our linear transformation from the Transformer row to the vocabulary scores, we have by \citet{foster2019vector}:
\begin{align*}
    &\Ev{\sup_{W, Y} \sum_{i=1}^{m}\epsilon_i\ell_{i}\p{W(Y_{L}^{(i)})_\tau}} \leq\\
    &\tilde{\mathcal{O}}(\sqrt{K})\max_{s} \Ev{\sup_{W, Y} \sum_{i=1}^{m}\epsilon_i\p{W(Y_{L}^{(i)})_\tau}_s} =\\
    &\tilde{\mathcal{O}}(\sqrt{K})\max_{s} \Ev{\sup_{W_s,Y} \sum_{i=1}^{m}\epsilon_iW_s(Y_{L}^{(i)})_\tau}
\end{align*}

Notice that since $W_s \in \R^{d}$, the resulting Rademacher complexity is of the same form of the scalar problem we found bounds for in the previous sections. Thus, we can apply those bounds to this set up.

\section{EMPIRICAL EXAMPLE: SPARSE MAJORITY}

The above sections show that, with bounded norms on the weight matrices, our generalization gap should not grow with sequence length. Thus, in this section, we will discuss a simulated study to see empirically if we find results that match our theory. We run a single layer Transformer on a simulated sparse majority data set on a variety of sequence lengths and we look at three results: (1) The total $1$-norm of the weights in the Transformer, (2) The cross entropy generalization gap of the best epoch, (3) The validation accuracy of the best epoch for each sequence length.

The first two will show whether or not our theoretical findings are found in practice as well. The last one is more of a practical concern--a situation where the generalization gap is small but the network does not learn is not very useful.

The dataset we create is a sequence of zeros and ones where the label is determined by a majority of a sparse set of the indices. More concretely, if we have a sequence $S$ of length is $T$, we have a set of indices is $I$, $|I| < T$, the label is 
\[y_i = \mathbbm{1}_{\left\{\sum_{i \in I}S_i > \frac{|I|}{2}\right\}}\]

In order to more accurately emulate real uses of Transformers, we embed $0$ and $1$ each into a $d$-dimensional vector where these two are orthogonal from each other. We also add the positional encoding defined by \citet{vaswani2017attention} to add positional information to the sequence.

For our experiment we used a single layer of Tensorflow's {\tt MultiHeadAttention} layer along with a second layer that extracts the $[CLS]$ layer and linearly transforms it into a vector of size 2. The loss we use is the cross entropy loss. Notice how we can use Section \ref{mask} to make this fit in the regime we have been discussing in this paper; we can act as if the $[CLS]$ index is always masked and we have a vocabulary of size $2$ ($0$ and $1$). 

The multihead attention layer has embedding dimension of $64$ and $2$ heads. The embedding dimension was chosen to be large while still allowing for moderate computation time. Only two heads was also chosen as well for computation time concerns. For our dataset, we had the sparse index set cardinality to $9$ and used $300$ training samples on sequence length $20$, $40$, $60$, \dots, $200$ with a validation set size of $10000$. The index set cardinality and training set size were chosen after finding a small enough size where the smaller sequence lengths could not always get perfect validation accuracy. 

The Transformer trained on a NVIDIA Tesla V100 GPU for $200000$ epochs with a batch size of $128$. $200000$ epochs, while a lot, was needed to allow for the larger sequence lengths to start to overfit. This batch size was also chosen after trial and error.

We trained our model for each sequence length $5$ times (new data set each time) and recorded the $1$-norms of the weights, the accuracy, and generalization gap of the best epoch. Figure \ref{gen_gap} shows the worst generalization gap for each sequence length. Figure \ref{norm} shows the largest $1$-norm of the weights of each sequence length. Figure \ref{acc} shows the best accuracy for each sequence length. We note that each of these do not necessarily represent the same run per sequence length.

\begin{figure}[!h]
\centering
\includegraphics[width=8cm]{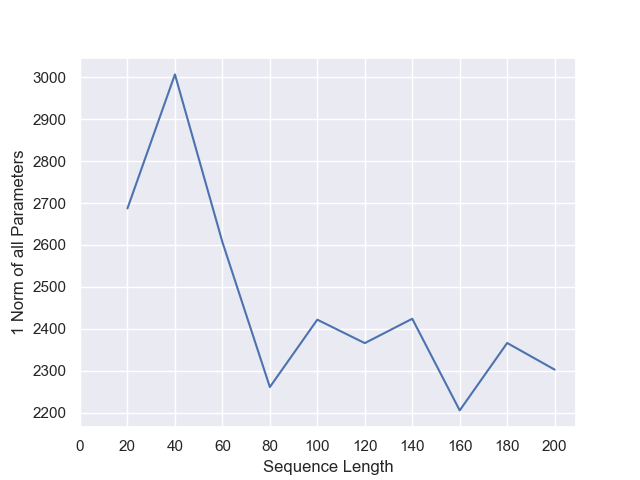}
\caption{Plot of the max sum of the absolute value of all the weights across sequence lengths. The lack of any trend further validates our assumption of bounded weights being credible.}
\label{norm}
\end{figure}

\begin{figure}[!h]
\centering
\includegraphics[width=8cm]{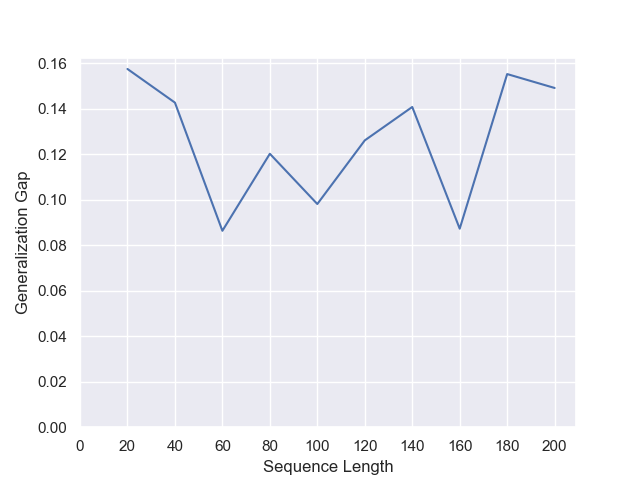}
\caption{Plot of the max generalization gap across sequence lengths. There is no discernible trend in the plot, giving empirical validation to our theoretical results}
\label{gen_gap}
\end{figure}

\begin{figure}[!h]
\centering
\includegraphics[width=8cm]{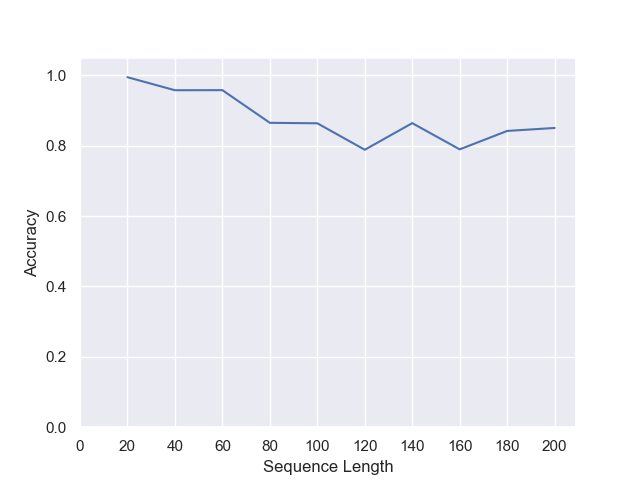}
\caption{Plot of the maximum accuracy across sequence lengths. While this paper makes no claims on the accuracy, we can see our accuracy does plateau. Therefore showing it is learning well at longer sequence lengths.}
\label{acc}
\end{figure}

 As we can see, the figure \ref{norm} shows that the weights do not increase with sequence length, lending strength to our matrix norm assumptions. 
 
 We can also see in figure \ref{gen_gap} the generalization gap also has no discernible trend and figure \ref{acc} shows the accuracy plateaus as sequence length increases. These results further help validate our theoretical findings that, surprisingly, large sequence lengths do not seem to affect how well Transformers learn.

The code for these experiments can be found \href{https://github.com/traugerjacob/Transformer-Gen-Bounds}{here}.

 \section{CONCLUSION AND FUTURE WORK}
 In this work, we give norm-based generalization bounds that do not grow with sequence length. This fills a hole in the literature where we can now have sequence length independent generalization bounds with the good properties the norm-based bounds give. We also give empirical evidence to validate our theoretical assumptions and theorems.

 Future work could include sharpening the linear covering number bounds and generalizing them for more types of matrix/input norm bound combinations. Another avenue could be analyzing exactly how the norm-based bounds and parameter counting bounds trade off with each other. 
 
\section{ACKNOWLEDGEMENTS}
 This research was supported in part through computational resources and services provided by Advanced Research Computing at the University of Michigan, Ann Arbor.

\bibliographystyle{apalike}
\bibliography{bib.bib}


\onecolumn
\begin{center}
\LARGE{\bf Appendix} 
\end{center}
\appendix

\section{Conversion from Scalar Valued Linear Covering Number Bound to Linear Transformation Covering Number Bound} \label{1d to nd}
Suppose we have sets $\mathcal{X}, \mathcal{M} \subset \R^d$, where $\forall w \in \mathcal{M}$, $\norm{w}_r \leq B_w$ and $\forall x \in \mathcal{X}$, $\norm{x}_s \leq B_x$ for some positive values $r$ and $s$. Suppose we also have a function class $\mathcal{F} = \{x \rightarrow w^\top x \mid w \in \mathcal{W}\}$ and a log covering number $C$ for this function class on inputs $\{x_i\}_{i=1}^{n} \subset \mathcal{X}$. Let $\mathcal{W} \subset \R^{k \times d}$, where $\forall W \in \mathcal{W}$, $\forall i \in [k]$, $\norm{W_i}_r \leq B_w$. 

Now, given a $W \in \mathcal{W}$, let us choose
\[\hat{W} = \begin{bmatrix}
    \hat{W}_1 \\ \hat{W}_2 \\ \vdots \\ \hat{W}_k
\end{bmatrix} \in \mathcal{W}\]
where $\hat{W}_j^\top$ is the column vector that would be chosen to cover $W_j^\top$ in the scalar case. Then notice for a positive value $q$ and for any $t \in [n]$:
\begin{align*}
    \norm{(W - \hat{W})x_t}_q^q = \sum_{j=1}^{k}((W_j -\hat{W}_j)x_t)^q \leq k\epsilon^{q}
\end{align*}
Thus we can see 
\[\log N_{\infty}\p{\mathcal{F}, k^{1/q}\epsilon, N, \norm{\cdot}_q} \leq kC\]
Therefore, if $C$ does not rely on $n$, neither does this bound.

\section{Proof of Lemma \ref{equiv_n_d}} \label{proof_equiv_n_d}
Notice how the right hand side is a lower bound for the left hand side. Thus, we need to show $\mathcal{N}_\infty\p{\mathcal{F}, \epsilon, \{B_xe_1,\dots,B_xe_d\}, \norm{\cdot}_q} \geq \mathcal{N}_\infty\p{\mathcal{F}, \epsilon, N, \norm{\cdot}_q}$. To do this, let $\hat{\mathcal{F}}$ be a set of size  $\mathcal{N}_\infty\p{\mathcal{F}, \epsilon, \{B_xe_1,\dots,B_xe_d\}, \norm{\cdot}_q}$ such that it covers $\mathcal{F}$ on the set $\{B_xe_1,\dots,B_xe_d\}$ to size $\epsilon$. We claim $\hat{\mathcal{F}}$ also covers $\mathcal{F}$ over any set of size $N$ in our input space. Let $\hat{\mathcal{W}}$ refer to the matrices used in $\hat{\mathcal{F}}$. Let $W \in \mathcal{W}$ and let $\hat{W} \in \hat{\mathcal{W}}$ be the matrix we would choose to cover $\mathcal{W}$. Notice for any $\norm{x}_1 \leq B_x$, we have
    
    \begin{align*}
        &\norm{(W - \hat{W})x}_q = \norm{\sum_{i=1}^{d}(W-\hat{W})x_ie_i}_q \leq \sum_{i=1}^{d}|x_i|\norm{(W_i- \hat{W}_i)}_q \leq B_x\max_{i \in [d]}\p{\norm{(W_i - \hat{W}_i)}_q} =\\
        &\max_{i \in [d]}\p{\norm{(W - \hat{W})B_xe_i}_q} \leq \epsilon
    \end{align*}
    Thus for any set of $\{x_i\}_{i=1}^{N}$ where $\norm{x}_1 \leq B_x$ we have 
    \[\sup_{j \in [N]}\norm{(W - \hat{W})x_j}_q \leq \sup_{j \in [N]}\max_{i \in [d]}\p{\norm{(W - \hat{W})B_xe_i}_q} =  \max_{i \in [d]}\p{\norm{(W - \hat{W})B_xe_i}_q} \leq \epsilon \]
    which shows that it is also an upper bound and thus we have an equality.

\section{Proof of Single Layer Bound (Theorem \ref{single thoerem})}\label{proof_single_layer}

Before we start the analysis, for any vectors $v,u \in \R^d$, $\norm{v}_1 \leq B_v$, notice the following inequality: 
\[v^Tu \leq B_v \max_{j \in [d]}|e_ju| = \max_{j \in [d], s \in \{-1,1\}}se_ju\]

We will also need this lemma that can be found in Edelman et al. 2022
\begin{lemma}\label{softmax} (Corollary A.7 in Edelman et al. 2022)
    For $\theta_1, \theta_2 \in \R^p$, we have \[\norm{\text{softmax}(\theta_1) - \text{softmax}(\theta_2)}_1 \leq 2 \norm{\theta_1 - \theta_2}_\infty\] 
\end{lemma}

Using the fist inequality above we can see we get:
{\begin{align*}
&\Ev{\sup_{w, W_c, W_v, W_{QK}} \sum_{i=1}^{m}\epsilon_iw^TW_c^T\sigma\p{W_v^TX_{(i)}^T\text{softmax}\p{X_{(i)}W_{QK}x_{[CLS]}}}} \leq \\
&B_w\Ev{\sup_{s, j\in[d] W_c, W_v, W_{QK}} s\sum_{i=1}^{m}\epsilon_ie_j^{T}W_c^T\sigma\p{W_v^TX_{(i)}^T\text{softmax}\p{X_{(i)}W_{QK}x_{[CLS]}}}} \end{align*}
But now notice that $e_j^TW_c^T$ is also just a row vector. Thus we can use the inequality again to get:
\begin{align*}
&B_wB_{W_c}\Ev{\sup_{s, j\in[k] W_v, W_{QK}} s\sum_{i=1}^{m}\epsilon_ie_j^{T}\sigma\p{W_v^TX_{(i)}^T\text{softmax}\p{X_{(i)}W_{QK}x_{[CLS]}}}} \leq
\end{align*}
}%
Since $\sigma$ is applied elementwise, we can bring $e_j^T$ inside of the function. Also, since $\sigma(0)=0$, we can see that if we have $W_v$ be the zero matrix, we have $0$ in our function class. Therefore, we can get rid of the sign function by using the well known property of Rademacher complexities of:
\[Rad_m(\mathcal{F} \cup -\mathcal{F}, S) \leq 2Rad_m(\mathcal{F}, S)\]
This, along with the contraction inequality \cite{ledoux1991probability} allows us to upper bound the above by
{\[2B_wB_{W_c}L_{\sigma}\Ev{\sup_{j\in[k] W_v, W_{QK}}\sum_{i=1}^{m}\epsilon_ie_j^{T}W_v^TX_{(i)}^T\text{softmax}\p{X_{(i)}W_{QK}x_{[CLS]}}}\]
}%
Continuing using the first inequality in this section again, we see
{\begin{align*}
&2B_wB_{W_c}L_{\sigma}\Ev{\sup_{j\in[k] W_v, W_{QK}}\sum_{i=1}^{m}\epsilon_ie_j^{T}W_v^TX_{(i)}^T\text{softmax}\p{X_{(i)}W_{QK}x_{[CLS]}}} \leq \\
&2B_wB_{W_c}L_{\sigma}B_{W_v}\Ev{\sup_{s, j\in[d]}\sup_{ W_{QK}}\sum_{i=1}^{m}s\epsilon_ie_j^{T}X_{(i)}^T\text{softmax}\p{X_{(i)}W_{QK}x_{[CLS]}}}
\end{align*}
}%
Let us call the expectation above $E$. We will now use covering numbers and Dudley's integral to bound $E$ to get our final generalization bound. First, we will use our covering number bound at scale $\epsilon' = \epsilon/2B_x^2$. Specifically, we will show that a set with a log size of $\ln(2d) + \frac{4B_x^4C}{\epsilon^2}$ covers the function class in the expectation above. The $\ln(2d)$ comes from the fact that we will have to modify $\hat{\mathcal{W}}$ to have it work for us. We do this by using the covering function class
{\[S = \{X \rightarrow se_j^{T}X^T\text{softmax}\p{X\hat{W}_{QK}x_{[CLS]}} \mid s \in \{-1,1\}, j\in [d], \hat{W}_{QK} \in \hat{\mathcal{W}\}}\]
}%
Notice this allows for every $\hat{W}_{QK} \in \hat{\mathcal{W}}$, we have every combination of $s$ and $e_j$.

Now, using the lemma \ref{softmax} and the linear algebra results of $\norm{Pv} \leq \norm{P}_{2,\infty}\norm{v}_1$ and $\norm{X}_{2\rightarrow \infty} = \norm{X^\top}_{2,\infty}$, we get:
{\begin{align*}
    &\norm{se_j^{T}X_{(i)}^T\text{softmax}\p{X_{(i)}W_{QK}x_{[CLS]}} - se_j^{T}X_{(i)}^T\text{softmax}\p{X_{(i)}\hat{W}_{QK}x_{[CLS]}}} \leq\\
    &\norm{X_{(i)}^T\text{softmax}\p{X_{(i)}W_{QK}x_{[CLS]}} - X_{(i)}^T\text{softmax}\p{X_{(i)}\hat{W}_{QK}x_{[CLS]}}} \leq\\
    &\norm{X_{(i)}^{T}}_{2,\infty} \norm{\text{softmax}\p{X_{(i)}W_{QK}x_{[CLS]}} - \text{softmax}\p{X_{(i)}\hat{W}_{QK}x_{[CLS]}}}_1 \leq\\
    &2\norm{X_{(i)}^{T}}_{2,\infty}\norm{X_{(i)}W_{QK}x_{[CLS]} - X_{(i)}\hat{W}_{QK}x_{[CLS]}}_\infty \leq\\
    &2\norm{X_{(i)}^{T}}_{2,\infty}^2\norm{W_{QK}x_{[CLS]}- \hat{W}_{QK}x_{[CLS]}} \leq\\
    &2B_{X}^2\frac{\epsilon}{2B_{X}^2} =\epsilon
\end{align*}
}%
Therefore, we can see the log covering number of $S$ is $\ln(2d) + \frac{4B_x^4C}{\epsilon^2}$. Also, notice, due to the softmax in $S$, the largest value the function class can be is $B_x$. Then, using Dudley's integral we have a constant $c$ such that
\begin{align*}
    &E/c \leq \inf_{\delta \geq 0} \delta + \int_{\delta}^{B_x}\sqrt{\frac{\ln(2d)}{m} + \frac{\frac{4B_x^4C}{\epsilon^2}}{m}}d\epsilon \leq\\
    &\inf_{\delta \geq 0}\delta + (B_x - \delta)\sqrt{\frac{\ln(2d)}{m}} + \int_{\delta}^{B_x}\sqrt{\frac{\frac{4B_x^4C}{\epsilon^2}}{m}}d\epsilon \leq\\
    &\inf_{\delta \geq 0} \delta + (B_x - \delta)\sqrt{\frac{\ln(2d)}{m}} + \frac{2B_x^2\sqrt{C}}{\sqrt{m}}\ln(B_x/\delta)
\end{align*}
When $m>\ln(2d)$ standard analysis can find the minimum for $\delta$ is when
\[\delta = \frac{\frac{2B_x^2\sqrt{C}}{\sqrt{m}}}{1 - \sqrt{\frac{\ln(2d)}{m}}} = \frac{2B_x^2\sqrt{C}}{\sqrt{m} - \sqrt{\ln(2d)}}\]
Substituting this in and rearranging we get
\begin{align*}
    &\frac{\frac{2B_x^2\sqrt{C}}{\sqrt{m}}}{1 - \sqrt{\frac{\ln(2d)}{m}}} \p{1 - \sqrt{\frac{\ln(2d)}{m}}} + B_x\sqrt{\frac{\ln(2d)}{m}} + \frac{2B_x^2\sqrt{C}}{\sqrt{m}}\ln\left(\frac{B_x(\sqrt{m} - \sqrt{\ln(2d)})}{2B_x^2\sqrt{C}}\right) =\\
    &\frac{2B_x^2\sqrt{C}}{\sqrt{m}} + B_x\sqrt{\frac{\ln(2d)}{m}} + \frac{2B_x^2\sqrt{C}}{\sqrt{m}}\ln\left(\frac{\sqrt{m} - \sqrt{\ln(2d)}}{2B_x\sqrt{C}}\right)
\end{align*}
Thus, multiplying this by $c$ and substituting this in for $E$ gives us our desired result.
 
\section{Corollaries of Theorem \ref{single thoerem}}\label{coro_single}
\begin{corollary}
    Let us have the requirements needed for Theorem \ref{single thoerem} along with $\norm{x}_{1} \leq B_x$ and $\norm{W_{QK}}_{1,\infty} \leq B_{W_{QK}}$. Let 
    \[B = B_wB_{W_c}L_{\sigma}B_{W_v}\]
    and let 
    \[\alpha = 2B_{W_{QK}}\sqrt{d\log(2d+1)}\]
    Then we have our Transformer Rademacher complexity being less than
    \small\begin{align*}
        O\left(B\left(\frac{B_x^3\alpha}{\sqrt{m}}\left(1 + \ln\left(\frac{\sqrt{m} - \sqrt{\ln(2d)}}{B_x^2\alpha}\right)\right)+ B_x\sqrt{\frac{\ln(2d)}{m}}\right)\right)
    \end{align*}
\end{corollary}

\begin{corollary}
    Let us have the requirements needed for Theorem \ref{single thoerem} along with $\norm{x}_{1} \leq B_x$ and $\norm{W_{QK}}_{2,1} \leq B_{W_{QK}}$ Let 
    \[B = B_wB_{W_c}L_{\sigma}B_{W_v}\]
    and let 
    \[\alpha = 2B_{W_{QK}}\sqrt{2\log(d)}\]
    Then we have our Transformer Rademacher complexity being less than
    \small\begin{align*}
        \hat{O}\left(B\left(\frac{B_x^3\alpha}{\sqrt{m}}\left(1 + \ln\left(\frac{\sqrt{m} - \sqrt{\ln(2d)}}{B_x^2\alpha}\right)\right)+ B_x\sqrt{\frac{\ln(2d)}{m}}\right)\right)
    \end{align*}
    Where the $\hat{O}$ denotes normal $O$ but with some logarithms not containing $T,d,k$ being omitted from inside the formula.
\end{corollary}

\section{Proof of Multiple Layers Covering Number (Theorem \ref{multi-layer})}\label{proof_multi_layer}
We will first start with some useful lemmas stated in \citet{edelman2022inductive}. The proofs of these lemmas will not be reproduced for ease of reading.

\begin{lemma}\label{lagrange}(Lemma A.8 in Edelman et al. 2022)
    For $\epsilon, C_i,\beta_i \geq 0$, $i \in [n]$ the solution to 
    \[\min_{\epsilon_1,\dots, \epsilon_n}\sum_{i=1}^{n}\frac{C_i}{\epsilon_i^2}\]
    given
    \[\sum_{i=1}^{n}\beta_i\epsilon_i = \epsilon\]
    is $\frac{\gamma^3}{\epsilon^2}$ where 
    \[\gamma = \sum_{i=1}^{n}C_i^{1/3}\beta_i^{2/3}\]
    and 
    \[\epsilon_i = \frac{\epsilon}{\gamma}\p{\frac{C_i}{\beta_i}}^{1/3}\]
\end{lemma}
The proof is by using Lagrange multipliers.

\begin{lemma}\label{a15}(Lemma A.15 from Edelman et al. 2022)
Suppose $W^{1:i+1}$, $\hat{W}^{1:i+1}$ satisfy our norm bounds. Then we have
\begin{align*}
    &\norm{(g_{block}^{(i+1)}(X;W^{1:i+1}) - g_{block}^{(i+1)}(X;\hat{W}^{1:i+1}))^\top}_{2,\infty} \leq \\
    &\norm{(W_c^{(i)} - \hat{W}_c^{(i)})^\top \sigma\p{\Pi_{norm}\p{f(g_{block}^{(i)}(X;\hat{W}^{1:i});\hat{W}^{(i)})}}^\top}_{2,\infty} +\\
    &L_\sigma B_{c2} B_{v2}(1 + 4B_{QK2})\norm{(g_{block}^{(i)}(X;W^{1:i}) - g_{block}^{(i)}(X;\hat{W}^{1:i}))^\top}_{2,\infty} +\\
    &2L_\sigma B_{c2}B_{v2}\norm{(W_{QK}^{(i)} - \hat{W}_{QK}^{(i)})^\top g_{block}^{(i)}(X;\hat{W}^{1:i})^\top}_{2,\infty} +\\
    &L_\sigma B_{c2}\norm{(W_v^{(i)} - \hat{W}_v^{(i)})^\top g_{block}^{(i)}(X;\hat{W}^{1:i})^\top}_{2,\infty}
\end{align*}
\end{lemma}

\begin{lemma}\label{a16}(Lemma A.16 from Edelman et al. 2022)
    Given $W^{1:i+1}$, $\hat{W}^{1:i+1}, w, \hat{w}$, we have:
    \begin{align*}
        &\left|g_{scalar}(X;W^{1:L+1}, w) - g_{scalar}(X;\hat{W}^{1:L+1}, \hat{w})\right| \leq\\
        &\norm{w}\norm{g_{block}^{(L+1)}(X;W^{1:L+1})_{[CLS]} - g_{block}^{(L+1)}(X;\hat{W}^{1:L+1})_{[CLS]}} + \\
        &\norm{(w - \hat{w})^\top g_{block}^{(L+1)}(X;\hat{W}^{1:L+1})_{[CLS]}}
    \end{align*}
\end{lemma}
The main proof ideas behind the above two lemmas is to unroll them, then use some norm properties and the triangle inequality to split them up.

Now given these, we will prove the multiple layers covering number theorem.
Suppose we have our linear covering bound described in the theorem statement. Let $X_1,\dots, X_m$ be the inputs that is within our norm bounds. Let $\mathcal{W}_v^{(i)}, \mathcal{W}_c^{(i)}, \mathcal{W}_{QK}^{(i)}$ be the sets of all possible values for $W_v^{(i)}$, $W_c^{(i)}, W_{QK}^{(i)}$ respectively. Let $\hat{\mathcal{W}}_v^{(i)}$ cover the function class $\{x \rightarrow W_v^\top x \mid W_v \in \mathcal{W}_v^{(i)}, \norm{x}_2 \leq 1\}$, let $\hat{\mathcal{W}}_c^{(i)}$ cover the function class $\{x \rightarrow W_c^\top x \mid W_c \in \mathcal{W}_c^{(i)}, \norm{x}_2 \leq 1\}$ 
 and let $\hat{\mathcal{W}}_{QK}^{(i)}$ cover the function class $\{x \rightarrow W_{QK} x \mid W_{QK} \in \mathcal{W}_{QK}^{(i)}, \norm{x}_2 \leq 1\}$ except for $\hat{\mathcal{W}}_{QK}^{(1)}$, which covers the same function class, but with $\norm{x} \leq B_x$. Let all of these classes be covered with $mT$ points and let $\epsilon_v^{(i)}$, $\epsilon_c^{(i)}$, $\epsilon_{QK}^{(i)}$ be the resolution for each cover. Also, let $\hat{\mathcal{W}}$ cover $\{x \rightarrow w^\top x \mid w \in \mathcal{W}, \norm{x} \leq 1\}$ at resolution $\epsilon_w$. The exact value of these resolutions will be shown at the end.

 We will show that for any $\epsilon > 0$ and for any $W^{1:L+1}$ that satisfies our norm bounds that there exists a
 \[\hat{W}^{1:L+1} \in \hat{\mathcal{W}}_c^{(1)} \otimes \hat{\mathcal{W}}_v^{(1)} \otimes \hat{\mathcal{W}}_{QK}^{(1)} \otimes \dots \otimes \hat{\mathcal{W}}_c^{(L)} \otimes \hat{\mathcal{W}}_v^{(L)} \otimes \hat{\mathcal{W}}_{QK}^{(L)} \otimes \hat{\mathcal{W}}\]
such that 
 \[\left|g_{scalar}(X;W^{1:L+1}, w) - g_{scalar}(X;\hat{W}^{1:L+1}, \hat{w})\right| \leq \epsilon\]

 To start, we will use lemma \ref{a16} to get 
 \begin{align*}
     &\left|g_{scalar}(X;W^{1:L+1}, w) - g_{scalar}(X;\hat{W}^{1:L+1}, \hat{w})\right| \leq\\
     &\norm{w}\norm{g_{block}^{(L+1)}(X;W^{1:L+1})_{[CLS]} - g_{block}^{(L+1)}(X;\hat{W}^{1:L+1})_{[CLS]}} + \norm{(w - \hat{w})^\top g_{block}^{(L+1)}(X;\hat{W}^{1:L+1})_{[CLS]}} \leq\\
        &\norm{w}\norm{(g_{block}^{(L+1)}(X;W^{1:L+1}) - g_{block}^{(L+1)}(X;\hat{W}^{1:L+1}))^\top}_{2,\infty} + \epsilon_w
 \end{align*}
 Now, we use lemma \ref{a15} to see 
 \begin{align*}
     &\norm{(g_{block}^{(L+1)}(X;W^{1:i+1}) - g_{block}^{(L+1)}(X;\hat{W}^{1:i+1}))^\top}_{2,\infty} \leq \\
    &\norm{(W_c^{(L)} - \hat{W}_c^{(L)})^\top \sigma\p{\Pi_{norm}\p{f(g_{block}^{(L)}(X;\hat{W}^{1:L});\hat{W}^{(L)})}}^\top}_{2,\infty} +\\
    &L_\sigma B_{c2} B_{v2}(1 + 4B_{QK2})\norm{(g_{block}^{(L)}(X;W^{1:L}) - g_{block}^{(L)}(X;\hat{W}^{1:L}))^\top}_{2,\infty} + \\
    &2L_\sigma B_{c2}B_{v2}\norm{(W_{QK}^{(i)} - \hat{W}_{QK}^{(i)})^\top g_{block}^{(i)}(X;\hat{W}^{1:i})^\top}_{2,\infty} +\\
    &L_\sigma B_{c2}\norm{(W_v^{(L)} - \hat{W}_v^{(L)})^\top g_{block}^{(L)}(X;\hat{W}^{1:L})^\top}_{2,\infty}
 \end{align*}
 Notice that if $C^{(i)} \in \R^{d \times T}$, $i \in [m]$, $W \in \R^{k \times d}$
 \[\max_{i\in[m]}\norm{(W - \hat{W})C^{(i)}}_{2,\infty} = \max_{i\in [m], t \in [T]}\norm{(W - \hat{W})C^{(i)}_{t}}\]
 Therefore we can use our covering number bounds to bound the values in the $\norm{\cdot}_{2,\infty}$.

 Thus, we get 
 \begin{align*}
     &\norm{(g_{block}^{(L+1)}(X;W^{1:i+1}) - g_{block}^{(L+1)}(X;\hat{W}^{1:i+1}))^\top}_{2,\infty} \leq\\
     &\epsilon_c^{(L)} + L_\sigma B_{c2} \epsilon_v^{(L)} + 2L_\sigma B_{c2}B_{v2}\epsilon_{QK}^{(L)} + L_\sigma B_{c2} B_{v2}(1 + 4B_{QK2})\norm{(g_{block}^{(L)}(X;W^{1:L}) - g_{block}^{(L)}(X;\hat{W}^{1:L}))^\top}_{2,\infty}
 \end{align*}
 Now note how we can iteratively do this for $\norm{(g_{block}^{(L)}(X;W^{1:L}) - g_{block}^{(L)}(X;\hat{W}^{1:L}))^\top}_{2,\infty}$ until we have gotten to the base case of $g_{block}^{(1)}(X;\hat{W}^{1:1})) = X$. Thus, if we let 
 \[\alpha_i = \prod_{j=i+1}^{L}L_\sigma B_{c2} B_{v2}(1 + 4B_{QK2})\]
 we can see that 
\begin{align*}
    &\max_{i \in {m}}\left|g_{scalar}(X_i;W^{1:L+1}, w) - g_{scalar}(X_i;\hat{W}^{1:L+1}, \hat{w})\right| \leq \\
    &\epsilon_w + B_w\p{\sum_{i=2}^{L}\alpha_i(\epsilon_c^{(i)} + L_\sigma B_{c2} \epsilon_v^{(i)} + 2L_\sigma B_{c2}B_{v2}\epsilon_{QK}^{(i)})} + B_w\alpha_1\left(\epsilon_c^{(1)} + L_\sigma B_{c2} \epsilon_v^{(1)} + 2L_\sigma B_{c2}B_{v2}\epsilon_{QK}^{(1)}\right)
\end{align*}
We left the first layer outside of the sum so we can recall $\epsilon_{QK}^{(1)}$ has a different input bound than the rest. 
Now, we can use lemma \ref{lagrange} to get our desired sizes for our different $\epsilon$'s and get our covering number stated in the theorem.

\section{Other Corollaries of Theorem \ref{multi-layer}}\label{coro_multi_layer}
\begin{corollary}
     Suppose we have the norm bounds required in lemma \ref{1,1,inf} for each $W_c^{(i)}, W_{v}^{(i)}, W_{QK}^{(i)}, w$ and let the maximum be $B$. Let $B_x$ be the input bound. Suppose we also have the bounds needed for theorem \ref{multi-layer}. 
    Let \small\begin{align*}
        &\alpha_i = \prod_{j=i+1}^{L}L_\sigma B_{c2} B_{v2}(1 + 4B_{QK2})\\
        &\tau_i = \alpha_i^{2/3} + \p{2\alpha_iL_\sigma B_{c2}B_{v2}}^{2/3} + (\alpha_iL_\sigma B_{v2})^{2/3}\\
        &\gamma = \p{dB^2B_x^2\ln(2k+1)}^{1/3}\p{2L_\sigma B_{c2}B_{v2}\alpha_1B_w}^{2/3} + \\
        &\p{dB^2\ln(2k+1)}^{1/3}\p{1 + (B_wL_\sigma B_{v2})^{2/3}}\\
        &\eta = \p{dB^2\ln(2k+1)}^{1/3}\p{B_w^{2/3}\sum_{i=2}^{L}\tau_i}
    \end{align*}
    Then, the log covering number of $g_{scalar}^{L+1}$ is 
    \begin{align*}
        &\frac{(\gamma + \eta)^3}{\epsilon^2}
    \end{align*}
\end{corollary}

\begin{corollary}
     Suppose we have the norm bounds required in lemma \ref{1,2,1} for each $W_c^{(i)}, W_{v}^{(i)}, W_{QK}^{(i)}, w$ and let the maximum be $B$. Let $B_x$ be the input bound. Suppose we also have the bounds needed for theorem \ref{multi-layer}. 
    Let \small\begin{align*}
        &\alpha_i = \prod_{j=i+1}^{L}L_\sigma B_{c2} B_{v2}(1 + 4B_{QK2})\\
        &\tau_i = \alpha_i^{2/3} + \p{2\alpha_iL_\sigma B_{c2}B_{v2}}^{2/3} + (\alpha_iL_\sigma B_{v2})^{2/3}\\
        &\gamma = \p{B^2B_x^2\ln(dk)}^{1/3}\p{2L_\sigma B_{c2}B_{v2}\alpha_1B_w}^{2/3} + \\
        &\p{B^2\ln(dk)}^{1/3}\p{1 + (B_wL_\sigma B_{v2})^{2/3}}\\
        &\eta = \p{B^2\ln(dk)}^{1/3}\p{B_w^{2/3}\sum_{i=2}^{L}\tau_i}
    \end{align*}
    Then, the log covering number of $g_{scalar}^{L+1}$ is 
    \begin{align*}
        &\frac{(\gamma + \eta)^3}{\epsilon^2}
    \end{align*}
\end{corollary}
\end{document}